\documentclass[english]{IEEEtran}
\usepackage[T1]{fontenc}
\usepackage[latin9]{inputenc}
\usepackage{float}
\usepackage{graphicx}

\makeatletter

\providecommand{\tabularnewline}{\\}
\floatstyle{ruled}
\newfloat{algorithm}{tbp}{loa}
\providecommand{\algorithmname}{Algorithm}
\floatname{algorithm}{\protect\algorithmname}

\newcommand{\lyxaddress}[1]{
\par {\raggedright #1
\vspace{1.4em}
\noindent\par}
}

\newtheorem{theorem}{Theorem}
\newtheorem{proof}{Proof}
\newtheorem{proposition}{Proposition}
\newtheorem{definition}{Definition}
\newtheorem{corollary}{Corollary}

\makeatother

\usepackage{babel}
\begin{document}

\title{IDS: An Incremental Learning Algorithm for Finite Automata}

\author{Muddassar A. Sindhu, Karl Meinke}

\maketitle

\lyxaddress{School of Computer Science and Communication, Royal Institute of
Technology, 10044 Stockholm, Sweden.}
\begin{abstract}
We present a new algorithm IDS for incremental learning of deterministic
finite automata (DFA). This algorithm is based on the concept of distinguishing
sequences introduced in \cite{Angluin81}. We give a rigorous proof
that two versions of this learning algorithm correctly learn in the
limit. Finally we present an empirical performance analysis that compares
these two algorithms, focussing on learning times and different types
of learning queries. We conclude that IDS is an efficient algorithm
for software engineering applications of automata learning, such as
formal software testing and model inference.\end{abstract}
\begin{IEEEkeywords}
Online learning, model inference, incremental learning, learning in
the limit, language inference.
\end{IEEEkeywords}

\section{\label{sec:Introduction}Introduction}

In recent years, automata learning algorithms (aka. regular inference
algorithms) have found new applications in software engineering such
as \emph{formal verification} (e.g. \cite{Peled},\cite{Clarke},
\cite{Luecker2006} ) \emph{software testing} (e.g. \cite{Raffelt},\cite{MeinkeSindhu2011}
) and \emph{model inference} (e.g. \cite{Bohlin}). These applications
mostly centre around learning an abstraction of a complex software
system which can then be statically analysed (e.g. by model checking)
to determine behavioural correctness. Many of these applications can
be improved by the use of learning procedures that are \emph{incremental}.

An automata learning algorithm is incremental if: (i) it constructs
a sequence of hypothesis automata $H_{0},H_{1},...$ from a sequence
of observations $o_{0},o_{1},...$ about an unknown target automaton
$A$, and this sequence of hypothesis automata finitely converges
to $A$, and (ii) the construction of hypothesis $H_{i}$ can reuse
aspects of the construction of the previous hypothesis $H_{i-1}$
(such as an equivalence relation on states). The notion of convergence
in the limit, as a model of correct incremental learning originates
in \cite{Gold67}.

Generally speaking, much of the literature on automata learning has focussed on offline learning from a fixed pre-existing data set describing the target automaton. Other approaches, such as \cite{Angluin81} and \cite{Angluin87} have considered online learning, where the data set can be extended by constructing and posing new queries. However, little attention has been paid to incremental learning algorithms, which can be seen as a subclass of online algorithms where serial hypothesis construction using a sequence of increasing data sets is emphasized. The much smaller collection of known incremental algorithms includes the RPNI2 algorithm of \cite{Dupont96}, the IID algorithm of \cite{Parekh} and the algorithm of \cite{PoratFeldman1991}. However, the motivation for incremental learning from a software engineering perspective is strong, and can be summarised as follows: \begin{enumerate} \item to analyse a large software system it may not be feasible (or even necessary) to learn the entire automaton model, and \item the choice of each relevant observation $o_{i}$ about a large unknown software system often needs to be iteratively guided by analysis of the previous hypothesis model $H_{i-1}$ for efficiency reasons. \end{enumerate}

Our research into efficient \emph{learning-based testing} (LBT) for software systems (see e.g. \cite{Meinke04}, \cite{MeinkeNiu}, \cite{MeinkeSindhu2011}) has led us to investigate the use of distinguishing sequences to design incremental learning algorithms for DFA. Distinguishing sequences offer a rather minimal and flexible way to construct a state space partition, and hence a quotient automaton that represents a hypothesis \emph{$H$} about the target DFA to be learned. Distinguishing sequences were first applied to derive the ID online learning algorithm for DFA in \cite{Angluin81}.

In this paper, we present a new algorithm\emph{ incremental distinguishing sequences} (\emph{IDS}), which uses the distinguishing sequence technique for incremental learning of DFA. In \cite{MeinkeSindhu2011} this algorithm has been successfully applied to learning based testing of reactive systems with demonstrated error discovery rates up to 4000 times faster than using non-incremental learning. Since little seems to have been published about the empirical performance of incremental learning algorithms, we consider this question too.
The structure of the paper is as follows. In Section \ref{prel}, we review some essential mathematical preliminaries, including a presentation of Angluin\textquoteright{}s original ID algorithm, which is necessary to understand the correctness proof for \emph{IDS}. In Section \ref{corr}, we present two different versions of the \emph{IDS} algorithm and prove their correctness. These are called: (1)\emph{ prefix free IDS}, and (2) \emph{prefix closed IDS}. In Section \ref{exper}, we compare the empirical performance of our two \emph{IDS} algorithms with each other. Finally, in Section \ref{concl}, we present some conclusions and discuss future directions for research.
\subsection{Related Work}
Distinguishing sequences were first applied to derive the ID online learning algorithm for DFA in \cite{Angluin81}. The ID algorithm is not incremental, since only a single hypothesis automaton is ever produced. Later an incremental version IID of this algorithm was presented in \cite{Parekh}. Like the IID algorithm, our \emph{IDS} algorithm is incremental. However in contrast with IID, the \emph{IDS} algorithm, and its proof of correctness are much simpler, and some technical errors in \cite{Parekh} are also overcome. 
Distinguishing sequences can be contrasted with the complete consistent table approach to partition construction as represented by the well known online learning algorithm L{*} of \cite{Angluin87}. Unlike L{*}, distinguishing sequences dispose of the need for an equivalence oracle during learning. Instead, we can assume that the observation set \emph{$P$} contains a \emph{live complete} set of input strings (see Section \ref{IDalgsec} below for a technical definition). Furthermore, unlike L{*} distinguishing sequences do not require a complete table of queries before building the partition relation. In the context of software testing, both of these differences result in a much more efficient learning algorithm. In particular there is greater scope for using online queries that have been generated by other means (such as model checking). Moreover, since LBT is a black-box approach to software testing, then the use of an equivalence oracle contradicts the black-box methodology.
In \cite{Dupont96}, an incremental version RPNI2 of the RPNI offline learning algorithm of \cite{OncingaGarcia1991} and \cite{Lang1992} is presented. The RPNI2 algorithm is much more complex than \emph{IDS}. It includes a recursive depth first search of a lexicographically ordered state set with backtracking, and computation of a non-deterministic hypothesis automaton that is subsequently rendered deterministic. These operations have no counterpart in \emph{IDS}. Thus \emph{IDS} is easier to verify and can be quickly and easily implemented in practise. The incremental learning algorithm introduced in \cite{PoratFeldman1991} requires a lexicographic ordering on the presentation of online queries, which is less flexible than \emph{IDS}, and indeed inappropriate for software engineering applications.
\section{Preliminaries}\label{prel}
\subsection{Notations and concepts for DFA}
Let $\Sigma$ be any set of symbols then $\Sigma^{*}$ denotes the set of all finite strings over $\Sigma$ including the empty string $\lambda$. The length of a string $\alpha\in\Sigma^{*}$ is denoted by $\vert\alpha\vert$ and $\vert\lambda\vert=0$. For strings $\alpha,\beta\in\Sigma^{*}$ , $\alpha\beta$ denotes their concatenation.
For $\alpha,\beta,\gamma\in\Sigma^{*}$, if $\alpha=\beta\gamma$ then $\beta$ is termed a \emph{prefix }of $\alpha$ and $\gamma$ is termed a \emph{suffix }of $\alpha$. We let $Pref(\alpha)$ denote the prefix closure of $\alpha$, i.e. the set of all prefixes of $\alpha$. We can also apply prefix closure pointwise to any set of strings. The \emph{set difference operation }between two sets \emph{U }and \emph{V }denoted by $U-V$ is the set of elements of \emph{U }which are not members of \emph{V. }The \emph{symmetric difference }operation defined on pairs of sets is defined by $U\oplus V=(U-V)\cup(V-U)$.
A \emph{deterministic finite automaton }(DFA) is a quintuple $A=\langle\Sigma,Q,F,q_{0},\delta\rangle$, where: $\Sigma$ is the input alphabet, $Q$ is the state set, $F\subseteq Q$ is the set of final states, $q_{0}\in Q$ is the initial state and state transition function $\delta$ is a mapping $\delta:Q\times\Sigma\rightarrow Q$, and $\delta(q_{i},b)=q_{j}$ meaning when in state $q_{i}\in Q$ given input $b$ the automaton $A$ will move to state $q_{j}\in Q$ in one step. We extend the function $\delta$ to a mapping $\delta^{*}=Q\times\Sigma^{*}\rightarrow Q$ inductively defined by $\delta=(q,\lambda)=q$ and $\delta^{*}=(q_{i},b_{1},...,b_{n})=\delta(\delta^{*}(q,b_{1},...,b_{n-1}),b_{n})$. The \emph{language }$L(A)$ accepted by $A$ is the set of all strings $\alpha\in\Sigma^{*}$ such that $\delta^{*}(q_{0},\alpha)\in F$. As is well known a language $L\subseteq\Sigma^{*}$is accepted by DFA if and only if $L$ is \emph{regular}, i.e. $L$ can be defined by a regular grammar. A state $q\in Q$ is said to be \emph{live} if for some string $\alpha\in\Sigma^{*}$ , $\delta^{*}(q,\alpha)\in F$ , otherwise $q$ is said to be \emph{dead. }Given a distinguished\emph{ dead state $d_{0}$ }we define\emph{ string concatenation modulo the dead state $d_{0}$, $f:\Sigma^{*}\cup\{d_{0}\}\times\Sigma\rightarrow\Sigma^{*}\cup\{d_{0}\}$, }by $f(d_{0},\sigma)=d_{0}$ and $f(\alpha,\sigma)=\alpha.\sigma$ for $\alpha\in\Sigma^{*}$. This is function is used for automaton learning in Section \ref{corr}. Given any DFA $A$ there exists a minimum state DFA $A^{\prime}$ such that $L(A)=L(A^{\prime})$ and this automaton is termed the \emph{canonical} DFA for $L(A)$. A canonical DFA has one dead state at the most.
We represent DFA graphically in the usual way using state diagrams. States are represented by small circles labelled by state names and final states among them are marked by concentric double circles. The initial state is represented by attaching a right arrow $\rightarrow$ to it. The transitions between the states are represented by directed arrows from state of origin to the destination state. The symbol read from the origin state is attached to the directed arrow as a label. Fig \ref{fig:Target} shows state transition diagram of one such DFA.
\subsection{The \emph{ID }Algorithm}\label{IDalgsec}
Our \emph{IDS }algorithm is an incremental version of the \emph{ID }learning algorithm introduced in \cite{Angluin81}. The ID algorithm is an online learning algorithm for complete learning of a DFA that starts from a given live complete set $P\subseteq\Sigma^{*}$ of queries about the target automaton, and generates new queries until a state space partition can be constructed. Since the algorithmic ideas and proof of correctness of IDS are based upon those of ID itself, it is useful to review the ID algorithm here. Algorithm 1 presents the ID algorithm. Since this algorithm has been discussed at length in \cite{Angluin81}, our own presentation can be brief. A detailed proof of correctness of ID and an analysis of its complexity can be found in \cite{Angluin81}.
A finite set $P\subseteq\Sigma^{*}$ of input strings is said to be \textit{live complete} for a DFA $A$ if for every live state $q\in Q$ there exists a string $\alpha\in P$ such that $\delta^{*}(q_{0},\alpha)=q$. Given a live complete set $P$ for a target automaton $A$, the essential idea of the ID algorithm is to first construct the set $T'=P\cup\{f(\alpha,b)\vert(\alpha,b)\in P\times\Sigma\}\cup\{d_{0}\}$ of all one element extensions of strings in $P$ as a set of state names for the hypothesis automaton. The symbol $d_{0}$ is added as a name for the canonical dead state. This set of state names is then iteratively partitioned into sets $E_{i}(\alpha)\subseteq T'$ for $i=0,1,\ldots$ such that elements $\alpha$, $\beta$ of $T'$ that denote the same state in $A$ will occur in the same partition set, i.e. $E_{i}(\alpha)=E_{i}(\beta)$. This partition refinement can be proven to terminate and the resulting collection of sets forms a congruence on $T'$. Finally the ID algorithm constructs the hypothesis automaton as the resulting quotient automaton. The method used to refine the partition set is to iteratively construct a set $V$ of \textit{distinguishing strings}, such that no two distinct states of $A$ have the same behaviour on all of $V$.

\begin{algorithm}
\textbf{Input}: A live complete set $P\subseteq\Sigma^{*}$ and a
teacher DFA $A$ to answer membership queries $\alpha\in L(A)$.

\textbf{Output: }A DFA $M$ equivalent to the target DFA $A$.

\centering{}\begin{enumerate} \item \textbf{\footnotesize begin}{\footnotesize{} }{\footnotesize \par} \item \textbf{\footnotesize /}{\footnotesize /Perform Initialization }{\footnotesize \par}\item {\footnotesize $i=0,$ $v_{i}=\lambda,$ $V=\{v_{i}\}$ }{\footnotesize \par} \item {\footnotesize $P'=P\cup\{d_{0}\}$, $T=P\cup\{f(\alpha,b)\vert(\alpha,b)\in P\times\Sigma\},$ $T'=T\cup\{d_{0}\}$ }{\footnotesize \par}\item {\footnotesize Construct function $E_{0}$ for $v_{0}=\lambda,$ }{\footnotesize \par} \item {\footnotesize $E_{0}(d_{0})=\emptyset$ }{\footnotesize \par} \item {\footnotesize $\forall\alpha\in T$ }{\footnotesize \par} \item {\footnotesize \{ pose the membership query $``\alpha\in L(A)?"$ }{\footnotesize \par} \item \textbf{\footnotesize $\qquad$if}{\footnotesize{} the teacher's response is $yes$ }{\footnotesize \par} \item \textbf{\footnotesize $\qquad$then}{\footnotesize{} $E_{0}(\alpha)=\{\lambda\}$ }{\footnotesize \par} \item \textbf{\footnotesize $\qquad$else}{\footnotesize{} $E_{0}(\alpha)=\emptyset$ }{\footnotesize \par} \item \textbf{\footnotesize $\qquad$end if}{\footnotesize{} }{\footnotesize \par} \item {\footnotesize \} }{\footnotesize \par} \item {\footnotesize //Refine the partition of the set $T'$ }{\footnotesize \par} \item \textbf{\footnotesize while}{\footnotesize{} $(\exists\alpha,\beta\in P^{\prime}$ and $b\in\Sigma$ such that $E_{i}(\alpha)=E_{i}(\beta)$ but $E_{i}(f(\alpha,b))\not\neq E_{i}(f(\beta,b)))$ }{\footnotesize \par} \item \textbf{\footnotesize do}{\footnotesize{} }{\footnotesize \par} \item {\footnotesize $\qquad$Let $\gamma\in E_{i}(f(\alpha,b))\oplus E_{i}(f(\beta,b))$ }{\footnotesize \par} \item {\footnotesize $\qquad$$v_{i+1}=b\gamma$ }{\footnotesize \par} \item {\footnotesize $\qquad$$V=V\cup\{v_{i+1}\}$, $i=i+1$ }{\footnotesize \par} \item {\footnotesize $\qquad$$\forall\alpha\in T_{k}$ pose the membership query $"\alpha v_{i}\in L(A)?"$ }{\footnotesize \par}\item {\footnotesize $\qquad\qquad$\{ }{\footnotesize \par} \item \textbf{\footnotesize $\qquad$$\qquad$if}{\footnotesize{} the teacher's response is $yes$ }{\footnotesize \par} \item \textbf{\footnotesize $\qquad$$\qquad$then}{\footnotesize{} $E_{i}(\alpha)=E_{i-1}(\alpha)\cup\{v_{i}\}$ }{\footnotesize \par} \item \textbf{\footnotesize $\qquad\qquad$else}{\footnotesize{} $E_{i}(\alpha)=E_{i-1}(\alpha)$ }{\footnotesize \par} \item \textbf{\footnotesize $\qquad\qquad$end if}{\footnotesize{} }{\footnotesize \par} \item {\footnotesize $\qquad\qquad$\} }{\footnotesize \par} \item \textbf{\footnotesize end while}{\footnotesize{} }{\footnotesize \par} \item {\footnotesize //Construct the representation $M$ of the target DFA $A$. }{\footnotesize \par} \item {\footnotesize The states of $M$ are the sets $E_{i}(\alpha)$, where $\alpha\in T$ }{\footnotesize \par} \item {\footnotesize The initial state $q_{0}$ is the set $E_{i}(\lambda)$ }{\footnotesize \par} \item {\footnotesize The accepting states are the sets $E_{i}(\alpha)$ where $\alpha\in T$ and $\lambda\in E_{i}(\alpha)$ }{\footnotesize \par} \item {\footnotesize The transitions of $M$ are defined as follows: }{\footnotesize \par} \item {\footnotesize $\quad\forall\alpha\in P^{\prime}$ }{\footnotesize \par} \item \textbf{\footnotesize $\qquad$if}{\footnotesize{} $E_{i}(\alpha)=\emptyset$ }{\footnotesize \par} \item \textbf{\footnotesize $\qquad$then}{\footnotesize{} add self loops on the state $E_{i}(\alpha)$ for all $b\in\Sigma$ }{\footnotesize \par} \item \textbf{\footnotesize $\qquad$else}{\footnotesize{} $\forall b\in\Sigma$ set the transition $\delta(E_{i}(\alpha),b)=E_{i}(f(\alpha,b))$ }{\footnotesize \par} \item \textbf{\footnotesize $\qquad$end if}{\footnotesize{} }{\footnotesize \par} \item \textbf{\footnotesize end.}{\footnotesize{} }{\footnotesize \par}\end{enumerate}\caption{\label{IDalg}\textbf{The ID Learning Algorithm}}
\end{algorithm}
\begin{figure}
\includegraphics[width=1\columnwidth,height=0.16\paperheight]{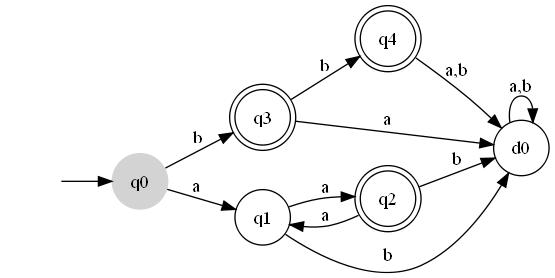}

\noindent \begin{centering}
\caption{\label{fig:Target}Target Automaton $A$}

\par\end{centering}

\end{figure}
\begin{figure}
\begin{centering}
\includegraphics[width=0.7\columnwidth,height=0.12\paperheight]{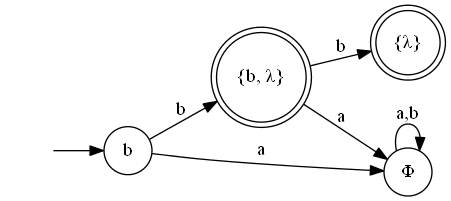}
\par\end{centering}

\caption{\label{fig:Hypothesis}Hypothesis Automaton $M_{1}$}

\end{figure}
\begin{algorithm}
\textbf{\footnotesize Input}{\footnotesize : A file $S=s_{1},\ldots,s_{l}$ of input strings $s_{i}\in\Sigma^{*}$ and a teacher DFA $A$ to answer membership queries $\alpha\in L(A)?$}\\ \textbf{\footnotesize Output:}{\footnotesize{} A sequence of $DFA$ $M_{t}$ for $t=0,\ldots,l$ as well as the total number of membership queries and book keeping queries asked by the learner. }{\footnotesize \par} \begin{enumerate} \item \textbf{\footnotesize begin}{\footnotesize{} }{\footnotesize \par} \item {\footnotesize $\quad$ }\textbf{\footnotesize /}{\footnotesize /Perform Initialization }{\footnotesize \par} \item {\footnotesize $\quad$ $i=0,$ $k=0,$ $t=0,$ $v_{i}=\lambda$, $V=\{v_{i}\}$ }{\footnotesize \par} \item {\footnotesize $\quad$ //Process the empty string }{\footnotesize \par} \item {\footnotesize $\quad$ $P_{0}=\{\lambda\}$, $P'_{0}=P_{0}\cup\{d_{0}\}$, $T_{0}=P_{0}\cup\Sigma$ }{\footnotesize \par} \item {\footnotesize $\quad$$E_{0}(d_{0})=\emptyset$ }{\footnotesize \par} \item {\footnotesize $\quad$ $\forall\alpha\in T_{0}$ $\{$ }{\footnotesize \par} \item {\footnotesize $\qquad$pose the membership query $``\alpha\in L(A)?"$,}{\footnotesize \par}\item {\footnotesize $\qquad$ $bquery=bquery+1$\label{bq1}}{\footnotesize \par}\item \textbf{\footnotesize $\qquad$if}{\footnotesize{} the teacher's response is $yes$ }{\footnotesize \par} \item \textbf{\footnotesize $\qquad$then}{\footnotesize{} $E_{0}(\alpha)=\{\lambda\}$ }{\footnotesize \par} \item \textbf{\footnotesize $\qquad$else}{\footnotesize{} $E_{0}(\alpha)=\emptyset$ }{\footnotesize \par} \item {\footnotesize $\quad$\} }{\footnotesize \par} \item {\footnotesize $\quad$//Refine the partition of set $T_{0}$ as described in Algorithm 3 }{\footnotesize \par} \item {\footnotesize $\quad$//Construct the current representation $M_{0}$ of the target $DFA$ }{\footnotesize \par} \item {\footnotesize $\quad$//as described in Algorithm 4. }{\footnotesize \par} \item {\footnotesize $\quad$ }{\footnotesize \par} \item {\footnotesize $\quad$//Process the file of examples. }{\footnotesize \par} \item {\footnotesize $\quad$}\textbf{\footnotesize while}{\footnotesize{} $(S\not=empty)$ }{\footnotesize \par}\item {\footnotesize $\quad$ \textbf{do} }{\footnotesize \par} \item {\footnotesize $\qquad$read( S, $\alpha$ ) }{\footnotesize \par} \item {\footnotesize $\qquad$mquery = mquery +1 \label{mq1}}{\footnotesize \par} \item {\footnotesize $\qquad$k = k+1, t = t+1 }{\footnotesize \par} \item {\footnotesize $\qquad$$P_{k}=P_{k-1}\cup\{\alpha\}\label{pf1}$ }{\footnotesize \par} \item {\footnotesize $\qquad$// $P_{k}=P_{k-1}\cup Pref(\alpha)\label{pc1}$ //prefix closure }{\footnotesize \par}\item {\footnotesize $\qquad$$P_{k}^{\prime}=P_{k}\cup\{d_{0}\}$ }{\footnotesize \par} \item {\footnotesize $\qquad$$T_{k}=T_{k-1}\cup\{\alpha\}\cup\{f(\alpha,b)\vert b\in\Sigma\}\label{pf2}$ }{\footnotesize \par} \item {\footnotesize $\qquad$// $T_{k}=T_{k-1}\cup Pref(\alpha) \cup\{f(\alpha,b)\vert\alpha\in P_{k}-P_{k-1},b\in\Sigma\}$\label{pc2}}{\footnotesize \par}\item {\footnotesize $\quad$ //Line \ref{pc2} for prefix closure}{\footnotesize \par} \item {\footnotesize $\qquad$$T'_{k}=T_{k}\cup\{d_{0}\}$ }{\footnotesize \par} \item {\footnotesize $\qquad$$\forall\alpha\in T_{k}-T_{k-1}$ }{\footnotesize \par} \item {\footnotesize $\qquad$$\quad$\{ }{\footnotesize \par} \item {\footnotesize $\qquad$$\qquad$// Fill in the values of $E_{i}(\alpha)$ using membership queries: }{\footnotesize \par} \item {\footnotesize $\qquad$$\qquad$$E_{i}(\alpha)=\{v_{j}\vert0\leq j\leq i,\alpha v_{j}\in L(A)\}$ }{\footnotesize \par} \item {\footnotesize $\qquad$$\qquad$$bquery=bquery+i$ \label{bq2} }{\footnotesize \par} \item {\footnotesize $\quad$$\qquad$\} }{\footnotesize \par} \item {\footnotesize $\qquad$// Refine the partition of the set $T_{k}$ }{\footnotesize \par} \item {\footnotesize $\qquad$}\textbf{\footnotesize if}{\footnotesize{} $\alpha$ is consistent with $M_{t-1}$ }{\footnotesize \par} \item {\footnotesize $\qquad$}\textbf{\footnotesize then}{\footnotesize{} $M_{t}=M_{t-1}$ }{\footnotesize \par} \item {\footnotesize $\qquad$}\textbf{\footnotesize else}{\footnotesize{} construct $M_{t}$ as described in Algorithm 4. }{\footnotesize \par} \item {\footnotesize $\quad$\textbf {end while} }{\footnotesize \par} \item \textbf{\footnotesize end.}{\footnotesize{} }{\footnotesize \par}\end{enumerate}

\caption{\textbf{\label{IDSalg}The IDS Learning Algorithm}}
\end{algorithm}
\begin{algorithm}
\begin{enumerate} \item \textbf{while} $(\exists\alpha,\beta\in P_{k}^{\prime}$ and $b\in\Sigma$ such that $E_{i}(\alpha)=E_{i}(\beta)$ but $E_{i}(f(\alpha,b))\not\neq E_{i}(f(\beta,b)))$  \item \textbf{\footnotesize do}{\footnotesize{} }{\footnotesize \par} \item {\footnotesize $\qquad$Let $\gamma\in E_{i}(f(\alpha,b))\oplus E_{i}(f(\beta,b))$ }{\footnotesize \par} \item {\footnotesize $\qquad$$v_{i+1}=b\gamma$ }{\footnotesize \par} \item {\footnotesize $\qquad$$V=V\cup\{v_{i+1}\}$, $i=i+1$ }{\footnotesize \par} \item {\footnotesize $\qquad$$\forall\alpha\in T_{k}$ pose the membership query $"\alpha v_{i}\in L(A)?"$ }{\footnotesize \par} \item {\footnotesize $\qquad\qquad$\{ }{\footnotesize \par} \item {\footnotesize $\qquad$$\qquad$$bquery=bquery+1$ }{\footnotesize \par} \item \textbf{\footnotesize $\qquad$$\qquad$if}{\footnotesize{} the teacher's response is $yes$ }{\footnotesize \par} \item \textbf{\footnotesize $\qquad$$\qquad$then}{\footnotesize{} $E_{i}(\alpha)=E_{i-1}(\alpha)\cup\{v_{i}\}$ }{\footnotesize \par} \item \textbf{\footnotesize $\qquad\qquad$else}{\footnotesize{} $E_{i}(\alpha)=E_{i-1}(\alpha)$ }{\footnotesize \par} \item \textbf{\footnotesize $\qquad\qquad$end if}{\footnotesize{} }{\footnotesize \par} \item {\footnotesize $\qquad\qquad$\} }{\footnotesize \par} \item \textbf{\footnotesize end while}{\footnotesize{} }{\footnotesize \par} \end{enumerate}

\caption{\label{refpartalg}\textbf{The Refine Partition Algorithm}}
\end{algorithm}
\begin{algorithm}
\begin{enumerate} \item {\footnotesize The states of $M_{t}$ are the sets $E_{i}(\alpha)$, where $\alpha\in T_{k}$ }{\footnotesize \par} \item {\footnotesize The initial state $q_{0}$ is the set $E_{i}(\lambda)$ }{\footnotesize \par} \item {\footnotesize The accepting states are the sets $E_{i}(\alpha)$ where $\alpha\in T_{k}$ and $\lambda\in E_{i}(\alpha)$ }{\footnotesize \par} \item {\footnotesize The transitions of $M_{t}$ are defined as follows: }{\footnotesize \par} \item {\footnotesize $\quad\forall\alpha\in P_{k}^{\prime}$ }{\footnotesize \par} \item \textbf{\footnotesize $\qquad$if}{\footnotesize{} $E_{i}(\alpha)=\emptyset$ }{\footnotesize \par} \item \textbf{\footnotesize $\qquad$then}{\footnotesize{} add self loops on the state $E_{i}(\alpha)$ for all $b\in\Sigma$ }{\footnotesize \par} \item \textbf{\footnotesize $\qquad$else}{\footnotesize{} $\forall b\in\Sigma$ set the transition $\delta(E_{i}(\alpha),b)=E_{i}(f(\alpha,b))$ }{\footnotesize \par} \item \textbf{\footnotesize $\qquad$end if}{\footnotesize{} }{\footnotesize \par} \item \textbf{\footnotesize $\quad$$\forall\beta\in T_{k}-P_{k}^{\prime}$}{\footnotesize \par} \item \textbf{\footnotesize $\qquad$if $\forall\alpha\in P_{k}^{\prime}$ $E_{i}(\beta)\neq E_{i}(\alpha)$ and $E_{i}(\beta)\neq\emptyset$}{\footnotesize \par} \item \textbf{\footnotesize $\qquad$then $\forall b\in\Sigma$ }{\footnotesize set the transition}\textbf{\footnotesize{} $\delta(E_{i}(\beta),b)=\emptyset$}{\footnotesize \par} \item \textbf{\footnotesize $\qquad$end if}{\footnotesize \par} \end{enumerate}

\caption{\textbf{\label{autoconstalg}The Automata Construction Algorithm}}
\end{algorithm}
\begin{table}
\begin{centering}
\begin{tabular}{|c|c|c|}
\hline 
$i$ & $0$ & $1$\tabularnewline
\hline 
\hline 
$v_{i}$ & $\lambda$ & $b$\tabularnewline
\hline 
$E(d_{0})$ & $\emptyset$ & $\emptyset$\tabularnewline
\hline 
$E(\lambda)$ & $\emptyset$ & $\{b\}$\tabularnewline
\hline 
$E(a)$ & $\emptyset$ & $\emptyset$\tabularnewline
\hline 
$E(b)$ & $\{\lambda\}$ & $\{\lambda,\; b\}$\tabularnewline
\hline 
$E(ba)$ & $\emptyset$ & $\emptyset$\tabularnewline
\hline 
$E(bb)$ & $\{\lambda\}$ & $\{\lambda\}$\tabularnewline
\hline 
\end{tabular}
\par\end{centering}

\caption{\label{Atable}For$(b,+)$}

\end{table}
\begin{figure}
\begin{centering}
\includegraphics[width=1\columnwidth,height=0.25\paperheight]{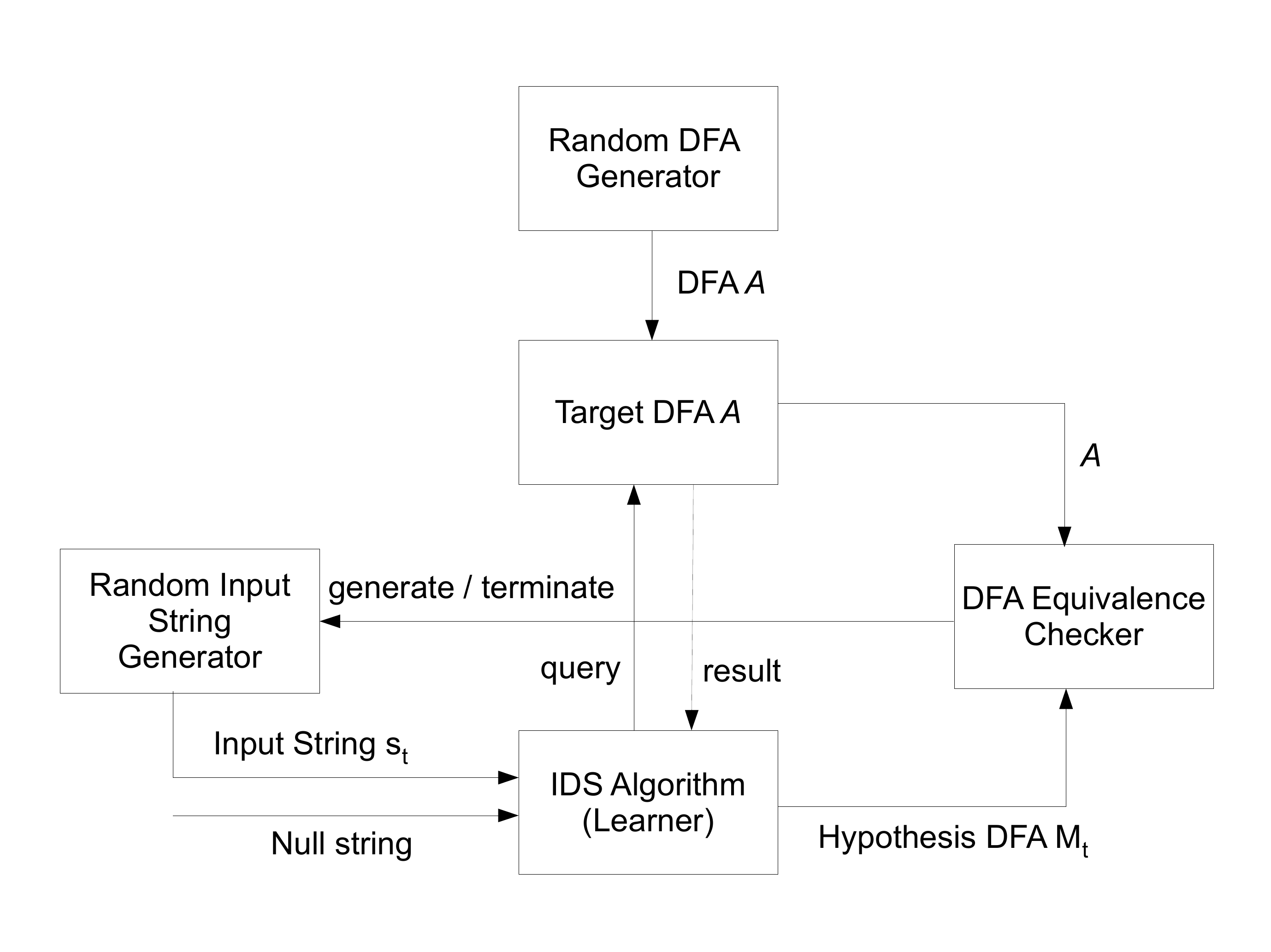}\caption{\label{evalFrame}Evaluation Framework}

\par\end{centering}

\end{figure}
\begin{figure}
\centering{}\includegraphics[width=1\columnwidth,height=0.25\paperheight]{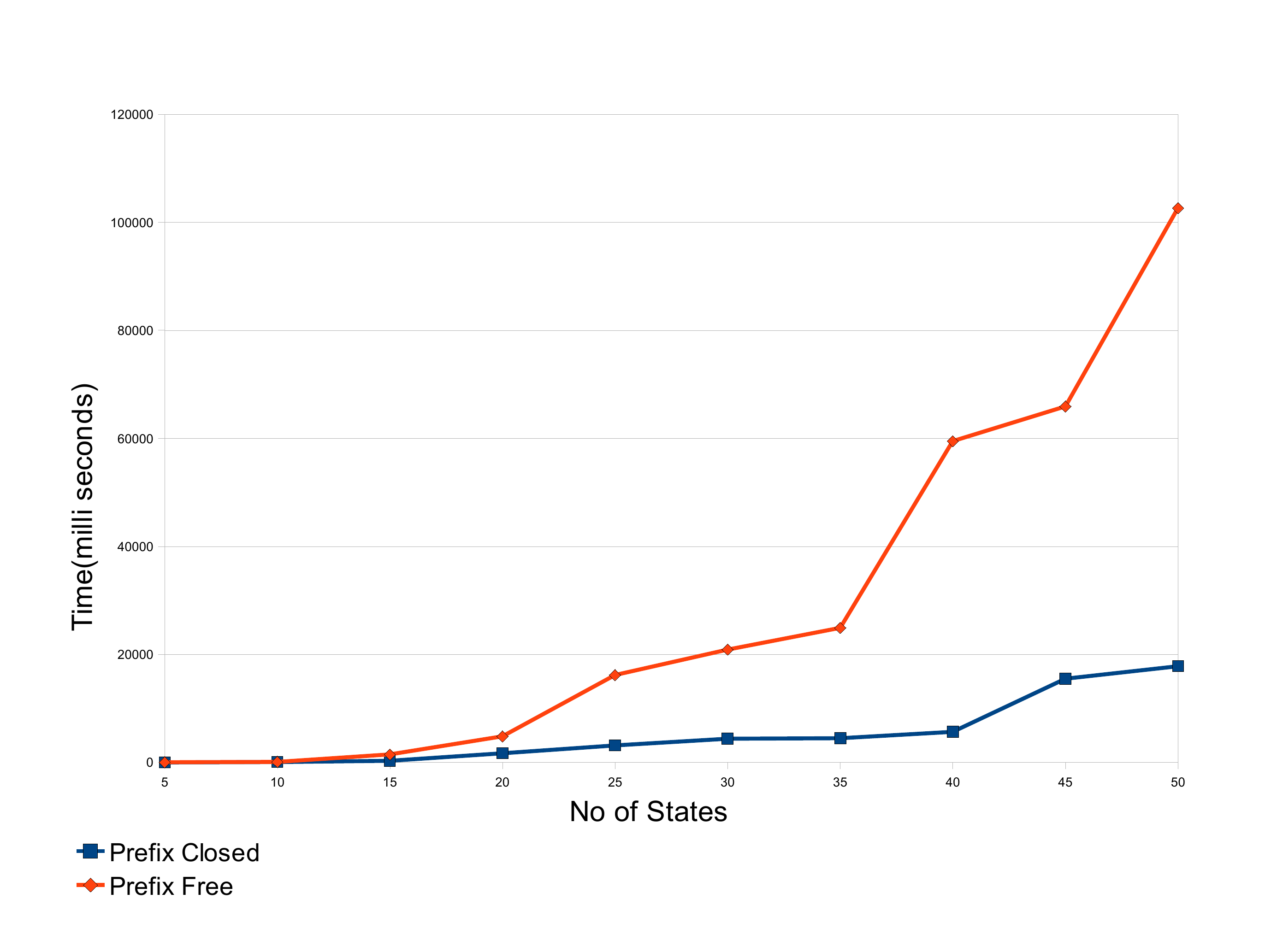}\caption{\label{fig:AvgTime}Average Time Complexity}
\end{figure}
\begin{figure}
\begin{centering}
\includegraphics[width=1\columnwidth,height=0.25\paperheight]{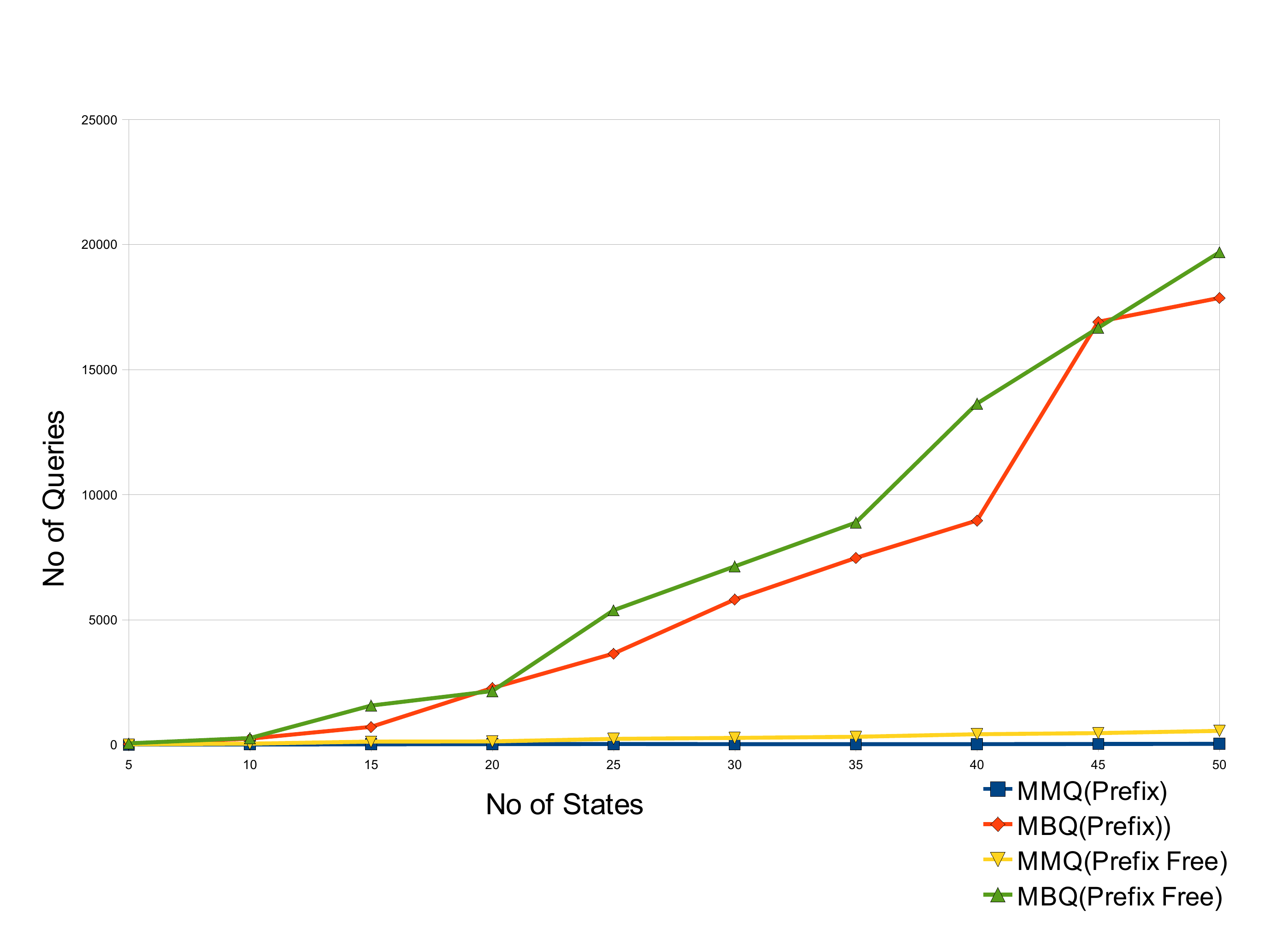}
\par\end{centering}

\caption{\label{fig:AvgQuery}Average Membership/Book-keeping Queries}
\end{figure}

We will present the ID and \emph{IDS} algorithms so that similar variables share the same names. This pedagogic device emphasises similarity in the behaviour of both algorithms. However, there are also important differences in behaviour. Thus, when analysing the behavioural properties of program variables we will carefully distinguish their context as e.g. $v_{n}^{ID}$, $E_{n}^{ID}(\alpha),\ldots$, and $v_{n}^{IDS}$, $E_{n}^{IDS}(\alpha),\ldots$ etc. Our proof of correctness for \emph{IDS} will show how the learning behaviour of \emph{IDS} on a sequence of input strings $s_{1},\ldots,s_{n}\in\Sigma^{*}$ can be simulated by the behaviour of ID on the corresponding set of inputs $\{s_{1},\ldots s_{n}\}$. Once this is established, one can apply the known correctness of ID to establish the correctness of \emph{IDS}. 
\subsection{Behavioural differences between IID and IDS}
The IID algorithm of \cite{Parekh} also presents a simulation method for ID. However following points of difference in behaviour of IID and \emph{IDS }are worth mentioning:  \begin{enumerate} \item IID starts from a null DFA as hypothesis while \emph{IDS }constructs the initial hypothesis after reading all $\sigma\in\Sigma$ from the initial state. \item IID discards all negative examples and waits for the first positive example after the construction of the initial (null) hypothesis to do further construction. \emph{IDS }on the other hand does construction with all negative and positive examples after building an initial hypothesis which makes it more useful for practical software engineering applications identified in Section \ref{sec:Introduction}.\item IID in some cases builds hypotheses which have a partially defined transition function $\delta$ rather than being \emph{left total}. This will be shown with an example in the next section. \emph{IDS} fixes this problem due to lines 10-13 of Algorithm \ref{autoconstalg} described in this paper. \item Unlike \emph{IDS}, there is no prefix free version of IID. \item In addition to the above it is easily shown that IID does not satisfy our Simulation Theorem \ref{thm:Let--be}, and thus the two algorithms are quite different. The behavioural properties of ID that are needed to complete this correctness proof can be stated as follows.\end{enumerate}
\begin{theorem} \textsl{\label{thm:(i)-Let-}(i) Let $P\subseteq\Sigma^{*}$ be a live complete set for a DFA $A$ containing $\lambda$. Then given $P$ and $A$ as input, the ID algorithm terminates and the automaton $M$ returned is the canonical automaton for $L(A)$.}  
\textsl{ (ii) Let $l\in{N}$ be the maximum value of program variable $i^{ID}$ given $P$ and $A$. For all $0\leq n\leq l$ and for all $\alpha\in T$,\[ E_{n}^{ID}(\alpha)=\{v_{j}^{ID}\;\vert\;0\leq j\leq n,\;\alpha v_{j}^{ID}\in L(A)\}.\]}
\end{theorem}
\begin{proof} (i) See \cite{Angluin81} Theorem 3. (ii) By induction on $n$.\end{proof}

One difference between ID and IID is the frequency of hypothesis automaton construction. With ID this occurs just once, after a single partition refinement is completed. However, with IID this occurs regularly, after each partition refinement. This difference means that the automaton construction algorithm (Algorithm \ref{IDalg}, lines 28-37) used for ID can no longer be used for IID, (as asserted in \cite{Parekh}) as we show below. Suppose we want to learn the automaton $A$ shown in Fig \ref{fig:Target}.
A positive example for this automaton is $(b,+)$. If we use it to start learning $A$ using the IID algorithm of \cite{Parekh}, we have $P_{0}=Pref(b)=\{b,\lambda\}$ and $P_{0}^{\prime}=\{b,\lambda\}\cup\{d_{0}\}=\{b,\lambda,d_{0}\}$. We then obtain the sets $T_{0}=P_{0}\cup\{f(\alpha,b)\vert(\alpha,b)\in P_{0}\times\Sigma\}=\{b,\lambda\}\cup\{b,\lambda\}\times\{a,b\}=\{\lambda,a,b,ba,bb\}$ and $T_{0}^{\prime}=T_{0}\cup\{d_{0}\}=\{d_{0},\lambda,a,b,ba,bb\}$. The initial column for the table of partition sets is constructed as shown in Table \ref{Atable}.

For $i=0$ and $v_{0}=\lambda$. From this column we see that two elements of the set $P_{0}^{\prime}$ have the same value. i.e $E(d_{0})=E(\lambda)$ but $E(f(d_{0},b))\neq E(f(\lambda,b))$. Therefore we have $\gamma\in E(f(d_{0},b))\oplus E(f(\lambda,b))$ or $\gamma\in\emptyset\oplus\{\lambda\}$. We can choose $\gamma=\lambda$ which gives the distinguishing string $v_{1}=b\gamma=b\lambda=b$. We then extend Table \ref{Atable} for $i=1$. Now we can see that all elements of the set $P_{0}^{\prime}$, which are $b,\lambda$ and $d_{0}$, have distinct values in the last column of the table and so no further refinement of the partition is possible. At this stage IID constructs the next hypothesis automaton $M_{1}$ as shown in Figure \ref{fig:Hypothesis}.

Now we observe that the transition function $\delta$ for this automaton is only partially defined, since there are no outgoing transitions for the state named $\{\lambda\}$. Therefore the IID algorithm of \cite{Parekh} does not always generate a hypothesis automaton with well defined transition function $\delta$. This created problems when we tried to use this algorithm for practical software engineering applications identified in Section \ref{sec:Introduction}, e.g. a model checker when used to verify this kind of a hypothesis can get stuck in such a state with no outgoing transitions. Similarly, an automata equivalence checker which is used to terminate a learning-based testing process goes into an infinite loop because of such a state since it will never find its equivalent state in the target automaton. 

The problem seems to stem from an \emph{unclosed} table in some executions of IID. The notion of \emph{closed }and \emph{consistent }observation table is given in \cite{Angluin87} for L{*} algorithm. The L{*} algorithm also incrementally builds the observation table but keeps asking queries until it becomes \emph{closed} and \emph{consistent}. In this case the table will be closed when $\forall\beta\in T\setminus P^{\prime}\;\exists\alpha\in P^{\prime}$ such that $E_{i}(\alpha)=E_{i}(\beta)$. If $E_{i}(\alpha)\neq E_{i}(\beta)$ as in the above example where $E_{1}(bb)=\{\lambda\}$ is not equal to any of $E_{1}(d_{0})=\emptyset,E_{1}(\lambda)=\{b\}\; or\; E_{1}(b)=\{\lambda,b\}$ then the solution in \cite{Angluin87} is to move $\beta\in T\setminus P^{\prime}$ to set $P^{\prime}$ and ask more queries to rebuild the congruence that might have been affected by this last addition to set $P^{\prime}$. However L{*} is a complete learning algorithm and it only outputs the description of the hypothesis automaton after learning it completely. Therefore for incremental learning the simplest fix is to set the transitions of such states to the dead state as done in lines 10-13 of Algorithm \ref{autoconstalg}. This doesn't require any new entries or shuffling the previous entries up in the table thus keeping the congruence intact.

\section{Correctness of \emph{IDS} Algorithm}\label{corr}
In this section we present our \emph{IDS} incremental learning algorithm for DFA. In fact, we consider two versions of this algorithm, with and without prefix closure of the set of input strings. We then give a rigorous proof that both algorithms correctly learn an unknown DFA in the limit in the sense of \cite{Gold67}.
In Algorithm \ref{IDSalg} we present the main \emph{IDS} algorithm, and in Algorithms \ref{refpartalg} and \ref{autoconstalg} we give its auxiliary algorithms for iterative partition refinement and automaton construction respectively. 
The version of the \emph{IDS} algorithm which appears in Algorithm \ref{IDSalg} we term the \textit{prefix free }\emph{IDS} algorithm, due to lines \ref{pf1} and \ref{pf2}. Notice that lines \ref{pc1} and \ref{pc2} of Algorithm \ref{IDSalg} have been commented out. When these latter two lines are uncommented and instead lines \ref{pf1} and \ref{pf2} are commented out, we obtain a version of the \emph{IDS} algorithm that we term \textit{prefix closed }\emph{IDS}. We will prove that both prefix closed and prefix free \emph{IDS} learn correctly in the limit. However, in Section \ref{exper} we will show that they have quite different performance characteristics in a way that can be expected to influence applications.

We will prove the correctness of the prefix free \emph{IDS} algorithm first, since this proof is somewhat simpler, while the essential proof principles can also be applied to verify the prefix closed \emph{IDS} algorithm. We begin an analysis of the correctness of prefix free \emph{IDS} by confirming that the construction of hypothesis automata carried out by Algorithm \ref{autoconstalg} is well defined.

\begin{proposition} \textsl{\label{pro:For-each-}For each $t\geq0$ the hypothesis automaton $M_{t}$ constructed by the automaton construction Algorithm \ref{autoconstalg} after $t$ input strings have been observed is a well defined DFA.} \end{proposition}
\begin{proof} The main task is to show $\delta$ to be well defined function and uniquely defined for every state $E_{i}(\alpha)$, where $\alpha\in T_{k}$. 

Proposition \ref{pro:For-each-} establishes that Algorithm \ref{IDSalg} will generate a sequence of well defined DFA. However, to show that this algorithm learns correctly, we must prove that this sequence of automata converges to the target automaton $A$ given sufficient information about $A$. It will suffice to show that the behaviour of prefix free \emph{IDS} can be simulated by the behaviour of ID, since ID is known to learn correctly given a live complete set of input strings (c.f. Theorem \ref{thm:(i)-Let-}.(i)). The first step in this proof is to show that the sequences of sets of state names $P_{k}^{IDS}$ and $T_{k}^{IDS}$ generated by prefix free $IDS$ converge to the sets $P^{ID}$ and $T^{ID}$ of ID.\end{proof} 

\begin{proposition} \textsl{\label{pro:Let--be}Let $S=s_{1},\ldots,s_{l}$ be any non-empty sequence of input strings $s_{i}\in\Sigma^{*}$ for prefix free $IDS$ and let $P^{ID}=\{\lambda,s_{1},\ldots,s_{l}\}$ be the corresponding input set for ID. } \end{proposition} \noindent \textsl{(i) For all $0\leq k\leq l$, $P_{k}^{IDS}=\{\lambda,s_{1},\ldots,s_{k}\}\subseteq P^{ID}$.}
\noindent \textsl{(ii) For all $0\leq k\leq l$, $T_{k}^{IDS}=P_{k}^{IDS}\cup\{f(\alpha,b)\vert\alpha\in P_{k}^{IDS},b\in\Sigma\}\subseteq T^{ID}$.}
\noindent \textsl{(iii) $P_{l}^{IDS}=P^{ID}$ and $T_{l}^{IDS}=T^{ID}$.} \begin{proof} Clearly (iii) follows from (i) and (ii). We prove (i) and (ii) by induction on $k$.  \end{proof}

Next we turn our attention to proving some fundamental loop invariants for Algorithm \ref{IDSalg}. Since this algorithm in turn calls the partition refinement Algorithm \ref{refpartalg} then we have in effect a doubly nested loop structure to analyse. Clearly the two indexing counters $k^{IDS}$ and $i^{IDS}$ (in the outer and inner loops respectively) both increase on each iteration. However, the relationship between these two variables is not easily defined. Nevertheless, since both variables increase from an initial value of zero, we can assume the existence of a monotone re-indexing function that captures their relationship.

\begin{definition}
Let $S=s_{1},\ldots,s_{l}$ be any non-empty sequence of strings $s_{i}\in\Sigma^{*}$. The re-indexing function ${K}^{S}:{N}\to{N}$ \emph{for prefix free} $IDS$ \emph{on input} $S$ \emph{is the unique monotonically increasing function such that for each} $n\in{N}$, ${K}^{S}(n)$ \emph{is the least integer} $m$ \emph{such that program variable} $k^{IDS}$ \emph{has value} $m$ \emph{while the program variable} $i^{IDS}$ \emph{has value} $n$. \emph{Thus, for example,} $K^{S}(0)=0$. \emph{When} $S$ \emph{is clear from the context, we may write} $K$ for $K^{S}$. 
\end{definition}

With the help of such re-indexing functions we can express important invariant properties of the key program variables $v_{j}^{IDS}$ and $E_{n}^{IDS}(\alpha)$, and via Proposition \ref{pro:Let--be} their relationship to $v_{j}^{ID}$ and $E_{n}^{ID}(\alpha)$. Corresponding to the doubly nested loop structure of Algorithm \ref{IDSalg}, the proof of Theorem \ref{thm:Let--be} below makes use of a doubly nested induction argument. 

\begin{theorem} \textbf{\textsl{\emph{(Simulation Theorem}}}\textsl{)\label{thm:Let--be} Let $S=s_{1},\ldots,s_{l}$ be any non-empty sequence of strings $s_{i}\in\Sigma^{*}$.}\textsl{\emph{ }}\textsl{For any execution of prefix free}\textsl{\emph{ }}IDS\textsl{ on}\textsl{\emph{ }}\textsl{S}\textsl{\emph{ }}\textsl{there exists an execution of ID on}\textsl{\emph{ }}\textsl{$\{\lambda,s_{1},\ldots,s_{l}\}$}\textsl{\emph{ }}\textsl{such that for all $m\geq0$:} \end{theorem}

\noindent \textsl{(i) For all $n\geq0$ if $K(n)=m$ then:}
\textsl{(a) for all $0\leq j\leq n$, $v_{j}^{IDS}=v_{j}^{ID}$,}
\textsl{(b) for all $0\leq j<n$, $v_{n}^{IDS}\not=v_{j}^{IDS}$,}
\textsl{(c) for all $\alpha\in T_{m}^{IDS}$, $E_{n}^{IDS}(\alpha)=\{v_{j}^{IDS}\vert0\leq j\leq n,\alpha v_{j}^{IDS}\in L(A)\}$.}
\noindent \textsl{(ii) If $m>0$ then let $p\in N$ be the greatest integer such that $K(p)=m-1$. Then for all $\alpha\in T_{m}^{IDS}$, $E_{p}^{IDS}(\alpha)=\{v_{j}^{IDS}\vert0\leq j\leq p,\alpha v_{j}^{IDS}\in L(A)\}$.}
\noindent \textsl{(iii) The mth partition refinement of }\emph{IDS}\textsl{ terminates.} 

\begin{proof} \noindent By induction on $m$ using Proposition \ref{pro:Let--be}(i).  \end{proof} \noindent Notice that in the statement of Theorem \ref{thm:Let--be} above, since both ID and \emph{IDS} are non-deterministic algorithms (due to the non-deterministic choice on line 17 of Algorithm \ref{IDalg} and line 3 of Algorithm \ref{refpartalg}), then we can only talk about the existence of some correct simulation. Clearly there are also simulations of \emph{IDS} by ID which are not correct, but this does not affect the basic correctness argument. 

\begin{corollary} \textsl{\label{cor:Let--be}Let $S=s_{1},\ldots,s_{l}$ be any non-empty sequence of strings $s_{i}\in\Sigma^{*}$. Any execution of prefix free IDS on $S$ terminates with the program variable $k^{IDS}$ having value $l$. } \end{corollary} \begin{proof} Follows from Simulation Theorem \ref{thm:Let--be}.(iii) since clearly the while loop of Algorithm \ref{IDSalg} terminates when the input sequence $S$ is empty.  \end{proof} \noindent Using the detailed analysis of the invariant properties of the program variables $P_{k}^{IDS}$ and $T_{k}^{IDS}$ in Proposition \ref{pro:Let--be} and $v_{j}^{IDS}$ and $E_{n}^{IDS}(\alpha)$ in Simulation Theorem \ref{thm:Let--be} it is now a simple matter to establish correctness of learning for the prefix free \emph{IDS} Algorithm.

\begin{theorem} \textbf{\textsl{\emph{\label{thm:(Correctness-Theorem)-Let}(Correctness Theorem)}}}\textsl{ Let $S=s_{1},\ldots,s_{l}$ be any non-empty sequence of strings $s_{i}\in\Sigma^{*}$ such that $\{\lambda,s_{1},\ldots,s_{l}\}$ is a live complete set for a DFA $A$. Then prefix free IDS terminates on $S$ and the hypothesis automaton $M_{l}^{IDS}$ is a canonical representation of $A$. } \end{theorem} \begin{proof} By Corollary \ref{cor:Let--be}, prefix free \emph{IDS} terminates on $S$ with the variable $k^{IDS}$ having value $l$. By Simulation Theorem \ref{thm:Let--be}.(i) and Theorem \ref{thm:(i)-Let-}.(ii), there exists an execution of ID on $\{\lambda,s_{1},\ldots,s_{l}\}$ such that $E_{n}^{IDS}(\alpha)=E_{n}^{ID}(\alpha)$ for all $\alpha\in T_{l}^{IDS}$ and any $n$ such that $K(n)=l$. By Proposition \ref{pro:Let--be}.(iii), $T_{l}^{IDS}=T^{ID}$ and $P_{l}^{\prime IDS}=P^{\prime ID}$. So letting $M^{ID}$ be the canonical representation of $A$ constructed by ID using $\{\lambda,s_{1},\ldots,s_{l}\}$ then $M^{ID}$ and $M_{l}^{IDS}$ have the same state sets, initial states, accepting states and transitions.  \end{proof}

\noindent Our next result confirms that the hypothesis automaton $M_{t}^{IDS}$ generated after $t$ input strings have been read is consistent with all currently known observations about the target automaton. This is quite straightforward in the light of Simulation Theorem \ref{thm:Let--be}. \begin{theorem} \textbf{\emph{(Compatibility Theorem)}} \textsl{Let $S=s_{1},\ldots,s_{l}$ be any non-empty sequence of strings $s_{i}\in\Sigma^{*}$. For each $0\leq t\leq l$, $M_{t}^{IDS}$ is compatible with $A$ on $\{\lambda,s_{1},\ldots,s_{t}\}$.} \end{theorem} \begin{proof} By definition, $M_{t}^{IDS}$ is compatible with $A$ on $\{\lambda,s_{1},\ldots,s_{t}\}$ if, and only if, for each $0\leq j\leq t$, $s_{j}\in L(A)\Leftrightarrow\lambda\in E_{i_{t}}^{IDS}(s_{j})$, where $i_{t}$ is the greatest integer such that $K(i_{t})=t$ and the sets $E_{i_{t}}^{IDS}(\alpha)$ for $\alpha\in T_{t}^{IDS}$ are the states of $M_{t}^{IDS}$. Now $v_{0}^{IDS}=\lambda$. So by Simulation Theorem \ref{thm:Let--be}.(i).(c), if $s_{j}\in L(A)$ then $s_{j}v_{0}^{IDS}\in L(A)$ so $v_{0}^{IDS}\in E_{i_{t}}^{IDS}(s_{j})$, i.e. $\lambda\in E_{i_{t}}^{IDS}(s_{j})$, and if $s_{j}\not\in L(A)$ then $s_{j}v_{0}^{IDS}\not\in L(A)$ so $v_{0}^{IDS}\not\in E_{i_{t}}^{IDS}(s_{j})$, i.e. $\lambda\not\in E_{i_{t}}^{IDS}(s_{j})$.  \end{proof}

\noindent Let us briefly consider the correctness of prefix closed \emph{IDS}. We begin by observing that the non-sequential ID Algorithm \ref{IDalg} does not compute any prefix closure of input strings. Therefore, Proposition \ref{pro:Let--be} does not hold for prefix closed \emph{IDS}. In order to obtain a simulation between prefix closed \emph{IDS} and ID we modify Proposition \ref{pro:Let--be} to the following. \begin{proposition} \textsl{Let $S=s_{1},\ldots,s_{l}$ be any non-empty sequence of input strings $s_{i}\in\Sigma^{*}$ for prefix closed $IDS$ and let $P^{ID}=Pref(\{\lambda,s_{1},\ldots,s_{l}\})$ be the corresponding input set for ID. } \end{proposition} \noindent \textsl{(i) For all $0\leq k\leq l$, $P_{k}^{IDS}=Pref(\{\lambda,s_{1},\ldots,s_{k}\})\subseteq P^{ID}$.}
\noindent \textsl{(ii) For all $0\leq k\leq l$, $T_{k}^{IDS}=P_{k}^{IDS}\cup\{f(\alpha,b)\vert\alpha\in P_{k}^{IID},b\in\Sigma\}\subseteq T^{ID}$.}
\noindent \textsl{(iii) $P_{l}^{IDS}=P^{ID}$ and $T_{l}^{IDS}=T^{ID}$} \begin{proof} Similar to the proof of Proposition \ref{pro:Let--be}.\end{proof} \begin{theorem} \textbf{\textsl{\emph{(Correctness Theorem)}}}\textsl{ Let $S=s_{1},\ldots,s_{l}$ be any non-empty sequence of strings $s_{i}\in\Sigma^{*}$ such that $\{\lambda,s_{1},\ldots,s_{l}\}$ is a live complete set for a DFA $A$. Then prefix closed $IDS$ terminates on $S$ and the hypothesis automaton $M_{l}^{IDS}$ is a canonical representation of $A$.}\end{theorem} \begin{proof} Exercise, following the proof of Theorem \ref{thm:(Correctness-Theorem)-Let}. \end{proof}

\section{Empirical Performance Analysis}\label{exper}
\noindent Little seems to have been published about the empirical performance and average time complexity of incremental learning algorithms for DFA in the literature. By the average time complexity of the algorithm we mean the average number of queries needed to completely learn a DFA of a given state space size. This question can be answered experimentally by randomly generating a large number of DFA with a given state space size, and randomly generating a sequence of query strings for each such DFA.
From the point of view of software engineering applications such as testing and model inference, we have found that it is important to distinguish between the two types of queries about the target automaton that are used by \emph{IDS} during the learning procedure. On the one, hand the algorithm uses internally generated queries (we call these \textit{book-keeping queries}) and on the other hand it uses queries that are supplied externally by the input file (we call these \textit{membership queries}). From a software engineering applications viewpoint it seems important that the ratio of book-keeping to membership queries should be low. This allows membership queries to have the maximum influence in steering the learning process externally. The average query complexity of the \emph{IDS} algorithm with respect to the numbers of book-keeping and membership queries needed for complete learning can also be measured by random generation of DFA and query strings. To measure each query type, Algorithm \ref{IDSalg} has been instrumented with two integer variables \emph{bquery} and \emph{mquery} intended to track the total number of each type of query used during learning (lines \ref{bq1}, \ref{mq1} and \ref{bq2}).
Since two variants of the \emph{IDS} algorithm were identified, with and without prefix closure of input strings, it was interesting to compare the performance of each of these two variants according to the above two average complexity measures.

To empirically measure the average time and query complexity of our two \emph{IDS} algorithms, two experiments were set up. These measured:\\  \indent (1) the average computation time needed to learn a randomly generated DFA (of a given state space size) using randomly generated membership queries, and\\  \indent (2) the total number of membership and book-keeping queries needed to learn a randomly generated DFA (of a given state space size) using randomly generated membership queries.
We chose randomly generated DFA with state space sizes varying between 5 and 50 states, and an equiprobable distribution of transitions between states. No filtering was applied to remove dead states, so the average effective state space size was therefore somewhat smaller than the nominal state space size.

The experimental setup consisted of the following components: \\  \indent {(1) a random input string generator},\\  \indent {(2) a random DFA generator},\\  \indent {(3) an instance of the \emph{IDS} Algorithm (prefix free or prefix closed) },\\  \indent {(4) an automaton equivalence checker}.

The architecture of our evaluation framework and the flow of data between these components are illustrated in Figure \ref{evalFrame}.
The random input string generator constructed strings over the set of alphabet $\Sigma$ of the target automaton and the length of the generated strings was always $\leq\vert Q\vert$ of the target automaton. Since \emph{IDS} begins learning by reading the null string, therefore, null string was only provided externally and wasn't generated randomly to avoid unnecessary repetition. The random DFA generator started by building a specific sized state set $Q$. The number of final states $\vert F\vert\leq\vert Q\vert$ was chosen randomly and then these final states were again marked randomly from the state set $Q$. The initial state was also chosen randomly from the state set. Similarly the transition function $\delta$ was constructed by randomly assigning next states from set $Q$ for each state $q\in Q$ after reading each alphabet $\sigma\in\Sigma$. The \emph{IDS} algorithms and the entire evaluation framework were implemented in Java. The performance of the input string and DFA generators is dependent on Java's Random class which generates pseudorandom numbers that depend upon a specific seed. To minimize the chance of generating the same pseudo random strings/automata again the seed was set to the system clock where ever possible.

The purpose of the equivalence checker was to terminate the learning procedure as soon as the hypothesis automaton sequence had successfully converged to the target automaton. There are several well known equivalence checking algorithms described in literature. These have runtime complexity ranging from quadratic to nearly linear execution times. We chose an algorithm with nearly linear time performance described in \cite{HopandKar}. This was to minimise the overhead of equivalence checking in the overall computation time.

Ten different automata were generated randomly for each state size (ranging between 5 and 50). They were learned by both prefix free and prefix closed \emph{IDS }and their learning times (in milli-seconds), number of book keeping queries and membership queries asked to reach the target were recorded. The graphs in Figure \ref{fig:AvgTime} and Figure \ref{fig:AvgQuery} show a mean of ten experiments for all these values for both variants of \emph{IDS.}
\subsection{Results and Interpretation}
The graphs in Figures \ref{fig:AvgTime} and \ref{fig:AvgQuery} illustrate the outcome of our experiments to measure the average time and average query complexity of both \emph{IDS} algorithms, as described in Section \ref{exper}.

Figure \ref{fig:AvgTime} presents the results of estimating the average learning time for the prefix free and prefix closed \emph{IDS} algorithms as a function of the state space size of the target DFA. For large state space sizes $\vert Q\vert$, the data sets of randomly generated target DFA represent only a small fraction of all possible such DFA of size $\vert Q\vert$. Therefore the two data curves are not smooth for large state space sizes. Nevertheless, there is sufficient data to identify some clear trends. The average learning time for prefix free \emph{IDS} learning is substantially greater than corresponding time for prefix closed \emph{IDS}, and this discrepancy increases with state space size. The reason would appear to be that prefix free \emph{IDS} throws away data about the target DFA that must be regenerated randomly (since input string queries are generated at random). The average time complexity for prefix free \emph{IDS} learning seems to grow approximately quadratically, while the average time complexity for prefix closed \emph{IDS} learning appears to grow almost linearly within the given data range. From this viewpoint, prefix-closed \emph{IDS} appears to be the superior algorithm.

Figure \ref{fig:AvgQuery} presents the results of estimating the average number of membership queries and book-keeping queries as a function of the state space size of the target DFA. Again, we have compared prefix-closed with prefix free \emph{IDS} learning. Allowance must also be made for the small data set sizes for large state space values. We can see that membership queries grow approximately linearly with the increase in state space size, while book-keeping queries grow approximately quadratically, at least within the data ranges that we considered. There appears to be a small but significant decrease in the number of both book-keeping and membership queries used by the prefix-closed \emph{IDS} algorithm. The reason for this appears to be similar to the issues identified for average time complexity. Prefix closure seems to be an efficient way to gather data about the target DFA. From the viewpoint of software engineering applications discussed in Section 1, now prefix free \emph{IDS} appears to be preferable. This is because the decreasing ratio of book-keeping to membership queries improves the possibility to direct the learning process using externally generated queries (e.g. from a model checker).

\section{Conclusions}\label{concl}
We have presented two versions of the \emph{IDS} algorithm which is an incremental algorithm for learning DFA in polynomial time. We have given a rigorous proof that both algorithms correctly learn in the limit. Finally we have presented the results of an empirical study of the average time and query complexity \emph{IDS}. These empirical results suggest that \emph{IDS} algorithm is well suited to applications in software engineering, where an incremental approach that allows externally generated online queries is needed. This conclusion is further supported in \cite{MeinkeSindhu2011} where we have evaluated the \emph{IDS} algorithm for learning based testing of reactive systems, and shown that it leads to error discovery up to 4000 times faster than using non-incremental learning. We gratefully acknowledge financial support for this research from the Higher Education Commission (HEC) of Pakistan, the Swedish Research Council (VR) and the EU under project HATS FP7-231620.

\bibliographystyle{IEEEtran}
\bibliography{bibfile}

\end{document}